\documentclass[submission]{eptcs}
\providecommand{\event}{TARK2023} % Name of the event you are submitting to

\usepackage{iftex}

\ifpdf
  \usepackage{underscore}         % Only needed if you use pdflatex.
  \usepackage[T1]{fontenc}        % Recommended with pdflatex
\else
  \usepackage{breakurl}           % Not needed if you use pdflatex only.
\fi

\usepackage{amssymb}
\usepackage{amsmath}
\usepackage{comment}
\usepackage{color}
\usepackage[ruled,vlined]{algorithm2e}
\usepackage{caption}
\usepackage{bussproofs}
\usepackage{multicol}
\usepackage{multirow}

\newenvironment{proof} {\textsc{Proof}\quad} {\hfill $\blacksquare$\\}
\definecolor{citecolor}{rgb}{0.5,0.5,0.5}
\newtheorem{theorem}{Theorem}
\newtheorem{definition}[theorem]{Definition}
\newtheorem{lemma}[theorem]{Lemma}
\newtheorem{proposition}[theorem]{Proposition}

\newcommand{\K}{\mathsf{K}}

\newcommand{\M}{\mathcal{M}}

\newcommand{\X}{\X}

\newcommand{\LAS}{\mathsf{L}_{AS}}
\newcommand{\SAS}{\mathbb{S}_{AS}}
\newcommand{\LEAS}{\mathsf{L}_{EAS}}
\newcommand{\SEAS}{\mathbb{S}_{EAS}}
\newcommand{\TermEAS}{Term^{EAS}(U)}
\newcommand{\TermNES}{Term^{NES}(U)}
\newcommand{\lr}[1]{\langle #1 \rangle}
\newcommand{\LNES}{\mathsf{L}_{NES}}

\newcommand{\STNES}{\mathbb{T}_{NES}}
\newcommand{\SFourNES}{\mathbb{S}4_{NES}}
\newcommand{\SFiveNES}{\mathbb{S}5_{NES}}
\renewcommand{\X}{\mathcal{X}}

\newcommand{\dK}{\mathsf{\widehat{K}}}

\newcommand{\us}[2]{\mathsf{All (#1, #2)}}
\newcommand{\es}[2]{\mathsf{Some (#1, #2)}}
\newcommand{\tf}{\mathsf{t}}
\newcommand{\ts}{\mathsf{g}}

\renewcommand{\phi}{\varphi}

\title{Epistemic Syllogistic: First Steps}

\author{ Yipu Li 
\institute{Department of Philosophy, Peking  University, CHINA}
\email{luxemburg@stu.pku.edu.cn}
\and  Yanjing Wang
\institute{Department of Philosophy, Peking  University, CHINA
}
\email{ y.wang@pku.edu.cn}
}

\def\titlerunning{Epistemic Syllogistic: First Steps}
\def\authorrunning{Li \& Wang}
\begin{document}
\maketitle

\begin{abstract}
Aristotle's discussions on modal syllogistic have often been viewed as error-prone and have garnered significant attention in the literature due to historical and philosophical interests. However, from a contemporary standpoint, they also introduced natural fragments of first-order modal logic, warranting a comprehensive technical analysis. In this paper, drawing inspiration from the natural logic program, we propose and examine several variants of modal syllogistic within the epistemic context, thereby coining the term \textit{Epistemic Syllogistic}. Specifically, we concentrate on the \textit{de re} interpretation of epistemic syllogisms containing non-trivial yet natural expressions such as ``all things known to be A are also known to be not B.'' We explore the epistemic apodeictic syllogistic and its extensions, which accommodate more complex terms. Our main contributions include several axiomatizations of these logics, with completeness proofs that may be of independent interest. 
\end{abstract}

\section{Introduction}
Although modal logic is regarded as a relatively young field, its origins can be traced back to Aristotle, who explored syllogistic reasoning patterns that incorporated modalities. However, in contrast to his utterly successful assertoric syllogistic, Aristotle's examination of modal syllogisms is often viewed as error-prone and controversial, thus receiving less attention from logicians. In the literature, a large body of research on Aristotle's modal syllogistic primarily centers on the possibility of a coherent interpretation of his proposed modal systems grounded by his philosophy on necessity and contingency (see, e.g., \cite{Rini11,Malink2013,Uckelman&Johnston10}). 

We adopt a more liberal view on Aristotle's modal syllogistic, considering it as a source of inspiration for formalizing natural reasoning patterns involving modalities, rather than scrutinizing the coherence of the original systems. Our approach is encouraged by the fruitful research program of \textit{natural logic}, which explores ``light'' logic systems that admit intuitive reasoning patterns in natural languages while balancing expressivity and computational complexity \cite{vanBenthem2008,NaturalLogic2015}. In particular, various extensions of the assertoric syllogistic have been proposed and studied \cite{NaturalLogic2015}. 

In this paper, we propose a systematic study on \textit{epistemic syllogistic} to initiate our technical investigations of (extensions of) modal syllogistic. The choice for the epistemic modality is intentional for its ubiquitous use in natural languages. Consider the following syllogism:  

\begin{center}
\AxiomC{All $C$ are $B$}
    \AxiomC{Some $C$ is \textit{known} to be $A$}
    \RightLabel{}
    \BinaryInfC{Some $B$ is \textit{known} to be $A$}
    \DisplayProof    
\end{center}

Taking the intuitive \textit{de re} reading, the second premise and the conclusion above can be formalized as $\exists x (Cx\land \K Ax)$ and $\exists x (Bx\land \K Ax)$ respectively in first-order modal logic (FOML).\footnote{The \textit{de dicto} reading of the second premise would be  $\K(\exists x (Cx\land Ax))$, which we do not discuss here.} It then becomes apparent that this syllogism is valid under the standard semantics of FOML. One objective of our investigation into epistemic syllogistic is to explore various natural fragments of FOML following  the general structure of syllogisms. 

Aristotle's original apodeictic syllogistic only allows a single occurrence of a necessity modality at a particular position in each sentence of assertoric syllogistic. However, from a modern perspective, we can greatly extend it and express interesting epistemic statements involving multiple agents and nested knowledge, such as ``Everything known to be $A$ by $i$ is also known to be $A$ by $j$''. Moreover, it is also interesting to allow nested knowledge such as ``Something $i$ knows that $j$ knows to be $A$ is also known to be $B$ by $i$''. The general idea is to extend the language of terms but keep the pattern of ``Some t is g'' and ``All t are g'', as proposed in \cite{Protin22}. 
%See Appendix~\ref{app.williamson} for a concrete example.

In this paper, we begin by presenting preliminaries about assertoric syllogisms in Section \ref{sec.pre}. We then proceed to examine the epistemic version of Aristotle's apodeictic syllogistic in Section~\ref{sec.eas} and provide a complete axiomatization. In Section \ref{sec.nes}, we significantly expand the language of terms in a compositional manner to allow for nesting of modalities with respect to multiple agents. The completeness of the proposed proof systems is demonstrated in Section \ref{sec.comp}. We conclude with a discussion of future work in the final section.

% \begin{itemize}
%     % \item Aristotle's Modal Syllogistic \cite{Protin22}
%     % \begin{itemize}
%     %     \item infamous modal syllogistic
%     % \end{itemize}
%     % \item Existing Work
%     % \begin{itemize}
%     %     \item Focus: justification\slash refutaion and connection with Aristotle's philosophy
%     %     \item Syntax and semantics
%     %     \item Results
%     % \end{itemize}
%     % \item Natural logic: take a liberal view 
%     % \begin{itemize}
%     % \item extensions of assertoric syllogisms
%     % \item adding relation but not modality
%     % \end{itemize}
%     % \item Epistemic reading
%     % \begin{itemize}
%     %     \item de re
%     %     \item de dicto
%     %     \item mixed
%     %     \item extension
%     % \end{itemize}
%    % \item Fragments of first-order epistemic logic 
%     % \begin{itemize}
%     %     \item know-wh and FO epistemic logic
%     %     \item Bundled fragments
%     %     \item The closest work: Aristotelian modal logic (\cite{Protin22})
%     % \end{itemize}
%     \item The contributions of the paper
%     \item Structure of the paper
% \end{itemize}

% Apodeictic syllogistic is concerned with
% necessity propositions and Aristotle focuses on four kinds of
% necessity propositions:

% \begin{quote}
% \begin{itemize}
%     \item A necessarily belongs to all B
% \item A necessarily belongs to no B
% \item A necessarily belongs to some B
% \item A necessarily does not belong to some B    
% \end{itemize}
% \end{quote}

%\noteYW{Check consistency of notation. }

\section{Preliminaries}
\label{sec.pre}
In this section, we familiarize the readers with the basics of Aristotle's syllogistic. Let us first consider the language of \textit{Assertoric Syllogistic}.
\begin{definition}[Language $\LAS$] Given a countable set of predicates $U$, the language of Assertoric Syllogistic is defined by the following grammar: 
%$$\begin{array}{clclcl}
% \varphi::=   &  \us{\tf}{\ts}\mid \es{\tf}{\ts} \\
% \tf::=     & A\\
% \ts::= &  A \mid \neg A\\
%\end{array}$$
\[\varphi ::= \us{\tf}{\ts} \mid \es{\tf}{\ts},\quad \tf ::= A,\quad \ts ::= A \mid \neg A
\]
where $A\in U.$ For the ease of presentation, we also write $\neg\us{A}{B} := \es{A}{\neg B}$, $\neg\us{A}{\neg B} := \es{A}{B}$, $\neg\es{A}{B} := \us{A}{\neg B}$ and $\neg\es{A}{\neg B} := \us{A}{B}$.
\end{definition}
The semantics for $\LAS$ is  based on first-order structures.
\begin{definition}[Semantics for $\LAS$]
A model of $\LAS$ is a pair $\mathcal{M} = (D,\rho)$ where $D$ is a non-empty domain and $\rho:U\to \mathcal{P}(D)$ is an interpretation function. The satisfaction relation is defined as below where the third column shows the equivalent clauses in the first-order language.
$$\begin{array}{|lclcl|}
\hline
\mathcal{M}\models_{AS} \us{A}{B}& \iff &\rho(A)\subseteq \rho(B) & &\mathcal{M}\Vdash\forall x (Ax\to Bx)\\
\mathcal{M}\models_{AS} \us{A}{\neg B} & \iff & \rho(A)\cap \rho(B) = \emptyset & &\mathcal{M}\Vdash\forall x (Ax\to \neg Bx)\\
\mathcal{M}\models_{AS} \es{A}{B} & \iff &\rho(A)\cap \rho(B)\neq \emptyset & &\mathcal{M}\Vdash\exists x (Ax\land Bx)\\
\mathcal{M}\models_{AS} \es{A}{\neg B} & \iff &\rho(A)\not\subseteq \rho(B)& &\mathcal{M}\Vdash\exists x (Ax\land \neg Bx)\\
\hline
\end{array}
$$
\end{definition}

Note that since we wish to generalize the ideas of the syllogistics from the modern perspective, the interpretation of a predicate can be an empty set, in contrast with the Aristotelian non-emptiness assumption. 

 Following the study of Corcoran \cite{Corcoran72} and Martin \cite{Martin97}, we present the following deduction system $\SAS$. Note that our system is slightly different from that of Corcoran's and Martin's, as they are loyal to Aristotle's non-emptiness assumption.\footnote{Cf. \cite{Moss11} for a direct proof system that replaces RAA rule by the explosion rule. Moss' work is targeted at a stronger language, which allows complement terms in the antecedent. e.g. $\us{\neg A}{\neg B}$.}

%\begin{multicols}{2}
 {\def\arraystretch{2.5}
\begin{tabular}{|l l l|}
\hline 
   $\us{A}{A}$\quad 
 &     \AxiomC{$\es{A}{B}$}
    \RightLabel{Conversion}
    \UnaryInfC{$\es{B}{A}$}
      \DisplayProof  &
                \AxiomC{[$\neg \phi$]}
    \UnaryInfC{$\psi$}
    \AxiomC{[$\neg \phi$]}
    \UnaryInfC{$\neg\psi$}
    \RightLabel{RAA}
    \BinaryInfC{$\phi$}
      \DisplayProof
      \\
       \AxiomC{$\us{A}{B}$}
    \AxiomC{$\us{B}{g}$}
    \RightLabel{Barbara-Celarent}
\BinaryInfC{$\us{A}{g}$}
  \DisplayProof  &     \AxiomC{$\es{A}{g}$}
    \RightLabel{Existence}
    \UnaryInfC{$\es{A}{A}$}
      \DisplayProof &\\
      \hline
\end{tabular}
}
\medskip

With a slight modification of Corcoran's result in Section 4 of \cite{Corcoran72}, it follows that the above system is sound and complete.
\begin{theorem}
    $\SAS$ is sound and strongly complete w.r.t. the semantics. 
\end{theorem}

\section{Epistemic Apodeictic Syllogistic}
\label{sec.eas}
Inspired by apodeictic syllogistic, we introduce the first language of \textit{Epistemic Syllogistic}.
%, which can be viewed as a modal extension of the syllogistic language with complements in \cite{Moss11}.
% \begin{itemize}
% \item Language and semantics
%     \item Expressivity: translation and bisimulation \noteYW{We can leave out the bisimulation part since it is not so interesting, even in the setting with nested modalities...}
%     \item Frame classes
%     \item Alternative semantics: abstract away the objects?
%     \item Axiomatization 
%     \item Decidability
% \end{itemize}

% \noteYW{Constant domain or increasing domains? Does it matter? }

\begin{definition}[Language $\LEAS$]Given a countable set of predicates $U$, the language of Epistemic Apodeictic Syllogistic is generated by the following grammar of formulas ($\phi$) and terms ($t, g$): 
% $$\begin{array}{cl}
%   \varphi::=   &  \us{\tf}{\ts}\mid \es{\tf}{\ts} \\
%     \tf::=     & A \\
%     \ts::= &  A \mid \neg A \mid  \K A \mid \K \neg A
% \end{array}$$
\[\varphi::= \us{\tf}{\ts}\mid \es{\tf}{\ts}, \quad  \tf::=     A, \quad     \ts::=   A \mid \neg A \mid  \K A \mid \K \neg A
\]
where $A\in U$. We collect all the $g$ as the set of (categorical) terms $\TermEAS$. 
\end{definition}
Note that the formulas should be read \textit{de re}. For example, $\us{A}{K\neg B}$ says ``all $A$ are known to be not $B$'', expressing the logical form $\forall x (Ax\to \K\neg Bx)$.  Formulas without modalities are called non-modal formulas. 
%Given a non-modal formula $\phi$, let $\K\phi$ denote its the formula with $\K$ attached to the consequent of $\phi$, e.g., $\K \us{A}{B}= \us{A}{KB}$.

$\LEAS$ is interpreted on first-order Kripke models with a constant domain. 
\begin{definition}[Models for $\LEAS$]
A model for $\LEAS$ a tuple $\mathcal{M} = (W, R, D, \rho)$. $W$ is the set of possible worlds, $R \subseteq W\times W$ is a reflexive relation, $D$ is the non-empty domain, and $\rho:W\times U\to \mathcal{P}(D)$ is the interpretation function. We also write $\rho_w(A)$ for $\rho(w,A)$. 
%A model is reflexive if for all $w\in W$, $wRw$. 
\end{definition}
Note that further frame conditions such as transitivity and Euclidean property do not play a role here since the syntax does not allow nested modalities, which will be relaxed in the next section.

To ease the presentation of the semantics, we extend the interpretation $\rho$ to any term. 
\begin{definition}
$\rho^+:W\times Term^{ES}(U)\to \mathcal{P}(D)$ is defined as:
$$ \rho_w^+(A) = \rho_w(A), \quad  \rho_w^+(\neg A) = D - \rho_w(A) \quad  \rho_w^+(\K A) = \bigcap_{wRv}\rho_v(A) \quad \rho_w^+(\K \neg A) = \bigcap_{wRv}(D - \rho_v(A))
$$
\end{definition}

\begin{definition}[Semantics for $\LEAS$]
Given a pointed model $\M,w$, the satisfaction relation is defined as follows where the third column lists the corresponding first-order modal formulas.
$$\begin{array}{|lclcl|}
\hline
\M,w\models_{ES} \us{A}{g} & \iff &\rho_w(A)\subseteq \rho_w^+(g)& &\M,w\Vdash\forall x (Ax\to g(x))\\
\M,w\models_{ES} \es{A}{g} &\iff & \rho_w(A)\cap \rho_w^+(g)\neq \emptyset
&&\M,w\Vdash\exists x (Ax\land g(x))\\
\hline
\end{array}
$$
where we abuse the notation and let $g(x)$ be a modal predicate formula defined as follows: 
$$
\begin{array}{l@{ \quad \text{ if } }lcl@{\quad  \text{ if } }l}
g(x)=Ax & g=A,   &   & g(x)=\neg Ax & g=\neg A  \\
g(x)=\K Ax & g=\K A,  & & g(x)=\neg \K Ax & g=\neg \K A \\
\end{array}
 $$
 where $\K$ is the modal operator and $Ax$ is an atomic formula.
\end{definition}

We propose the following proof system $\SEAS$:
\begin{center}
 {\def\arraystretch{2.5}
\begin{tabular}{|l l|}
\hline
   $\us{A}{A}$  
     &
    \AxiomC{[$\neg \phi$]}
    \UnaryInfC{$\psi$}
    \AxiomC{[$\neg \phi$]}
    \UnaryInfC{$\neg\psi$}
    \RightLabel{RAA (given non-modal $\varphi,\psi$)}
    \BinaryInfC{$\phi$}
      \DisplayProof \\
    \AxiomC{$\es{A}{Kg}$}
    \RightLabel{E-Truth}
    \UnaryInfC{$\es{A}{g}$}
    \DisplayProof & 
    \AxiomC{$\us{A}{Kg}$}
    \RightLabel{A-Truth}
    \UnaryInfC{$\us{A}{g}$}
    \DisplayProof \\
    \AxiomC{$\es{A}{B}$}
    \RightLabel{Conversion}
    \UnaryInfC{$\es{B}{A}$}
    \DisplayProof & 
    \AxiomC{$\us{A}{B}$}
    \AxiomC{$\us{B}{g}$}
    \RightLabel{Barbara/Celarent}
    \BinaryInfC{$\us{A}{g}$}
    \DisplayProof \\
    \AxiomC{$\es{A}{B}$}
    \AxiomC{$\us{B}{g}$}
    \RightLabel{Darii/Ferio}
    \BinaryInfC{$\es{A}{g}$}
    \DisplayProof &
    \AxiomC{$\us{C}{B}$}
    \AxiomC{$\es{C}{Kg}$}
    \RightLabel{Disamis/Bocardo}
    \BinaryInfC{$\es{B}{Kg}$}
    \DisplayProof \\
    \AxiomC{$\es{A}{g}$}
    \RightLabel{Existence 1}
    \UnaryInfC{$\es{A}{A}$}
    \DisplayProof & 
    \AxiomC{$\es{B}{KA}$}
    \RightLabel{Existence 2}
    \UnaryInfC{$\es{A}{KA}$}
    \DisplayProof\\ \hline
\end{tabular}
}
\end{center}
%\noteYW{Need to revise RAA: 
%For a non-modal sentence $\phi_\Box$, if $\Sigma\cup \{\phi_\Box\}$ is inconsistent, then $\Sigma\vdash_{ES} \neg \phi_\Box$.}
We say a set of formulas is \textit{consistent} if it cannot derive a contradiction in system $\SEAS$. Note that the RAA rule is restricted to non-modal formulas, as formulas with $\K$ in $\LEAS$ do not have negations expressible in the language. 
%As always, the definition of inconsistency is to be able to derive a pair of $\phi,\neg\phi$. But in ES, sentence with $K$ modality does not have negation. 
\begin{theorem}[Completeness]\label{thm.es}
    If $\Sigma\models_{ES} \phi$, then $\Sigma\vdash_{\SEAS} \phi$.
\end{theorem}
Due to the lack of space, we only sketch the idea of the (long) proof in Appendix \ref{app.thm.es}.
%\noteYW{Need to replace the proof with a proof sketch.}
% \begin{definition}
% $$\begin{array}{cl}
%   \varphi::=   &  \us{\ts}{\ts}\mid \es{\ts}{\ts} \\
%   \lt::= & A \mid \neg A \\
%     \ts::= &  \lt \mid  \K \lt 
% \end{array}$$
% \end{definition}
\section{Multi-agent Syllogistic with Nested Knowledge} \label{sec.nes}
The language $\LEAS$ has an asymmetry in the grammar such that the first term is simpler than the second. In this section, we restore the symmetry of the two terms. Moreover, the terms are now fully compositional using modalities and negations, thus essentially allowing nested modalities in both $\Box$ and $\Diamond$ shapes, also in a multi-agent setting. It can be viewed as a modal extension of the language of Syllogistic Logic with Complement in \cite{Moss11}, or a fragment of the language of Aristotelian Modal Logic in \cite{Protin22}. 
\begin{definition}[Language $\LNES$]
Given a countable set of predicates $U$ and a set of agents $I$, the language $\LNES$ is defined by the following grammar:
% $$\begin{array}{cl}
%   \varphi::=   &  \us{\ts}{\ts}\mid \es{\ts}{\ts} \\
%  %   t::=     & A \\
%     \ts::= &  A \mid \K g \mid \neg{g}
% \end{array}$$
\[ \varphi::=     \us{\ts}{\ts}\mid \es{\ts}{\ts}, \quad \ts::=   A \mid \K_i g \mid \neg{g}
\]
Where $A\in U$ and $i\in I$. The set of terms $g$ is denoted as $Term^{NES}(U)$.
\end{definition}
As before, we define $\neg\us{g_1}{g_2}:= \es{g_1}{\neg g_2}$ and $\neg\es{g_1}{g_2}:= \us{g_1}{\neg g_2}$. 
%For example, $\neg \us{A}{\neg B}$ is actually $\es{A}{\neg \neg B}$, which will be shown to be equivalent to $\es{A}{B}$. 
Moreover, let $\dK_i g$ be an abbreviation for $\neg \K_i \neg g$. With this powerful language $\LNES$, we can express the following: ``Everything $i$ knows to be $A$, $j$ also knows'' by $\us{\K_i A}{\K_j A}$; ``According to $i$, something known to be $B$ is possible to be also $A$'' by $\es{\K_i B}{\dK_i A}$; %``Everything $i$ does not know to be $A$ is known by $j$ that $i$ indeed does not know'' by $\us{\neg \K_i A}{\K_j \neg \K_i A}$. 
``Everything $i$ knows
that $j$ knows to be $A$ is also known to be $B$ by $i$'' by $\us{\K_i\K_j A}{\K_iB}.$
%As a concrete example, we formalize Williamson's argument against positive introspection \cite{Williamson02} in Appendix~\ref{app.williamson}.
%\begin{definition}[Double Negation]
    %Call $g_1,g_2$ double negation equivalent ($\sim_{\neg\neg}$) if they have the same predicate, same number of $K$, and between each adjacent $K$, the difference of number of negation is even. For instance, $K\neg\neg A\sim_{\neg\neg}KA$, $K\neg\neg K\neg B\sim_{\neg\neg}KK\neg\neg\neg B$. 
%\end{definition} 

$\LNES$ is also interpreted on first-order Kripke models with a constant domain and multiple relations $(W, \{R_i\}_{i\in I}, D, \rho)$. We say the model is a T\slash S4\slash S5 model if each $R_i$ is a reflexive\slash reflexive and transitive \slash equivalence relation, respectively. Now we define $\rho^+$, the interpretation function for terms. 
\begin{definition}
$\rho^+:W\times Term^{NES}(U)\to \mathcal{P}(D)$ is defined recursively as:
$$   \rho_w^+(A) = \rho_w(A) \qquad  \rho_w^+(\neg g) = D - \rho_w^+(g) \qquad \rho_w^+(\K_i g) = \bigcap_{wR_iv}\rho_v^+(g)
$$
\end{definition}
 It is easy to see that $\rho_w^+(\dK_i g) =\rho_w^+(\neg\K_i\neg g) =\bigcup_{wR_iv}\rho_v^+(g)$.
\medskip

% \begin{definition}[Semantics]

% For simplicity we write $\rho(w,A)$ as $\rho_w(A)$.

% An \textit{interpretation function for terms} is $\rho^+:W\times Term^{NES}(U)\to \mathcal{P}(D)$ s.t. 
% \begin{itemize}
%     \item $\rho_w^+(A) = \rho_w(A)$
%     \item $\rho_w^+(\neg g) = D - \rho_w^+(g)$
%     \item $\rho_w^+(\K g) = \bigcap_{wRv}\rho_v^+(g)$
% \end{itemize}
% \end{definition}
\begin{definition}[Semantics for $\LNES$]
The third column is the corresponding FOML formulas.
$$\begin{array}{|lclcl|}
\hline
\M,w\models_{NES} \us{g_1}{g_2} & \iff &\rho^+_w(g_1)\subseteq \rho_w^+(g_2)& &\M,w\Vdash\forall x (g_1(x)\to g_2(x))\\
\M,w\models_{NES} \es{g_1}{g_2} &\iff & \rho^+_w(g_1)\cap \rho_w^+(g_2)\neq \emptyset
&&\M,w\Vdash\exists x (g_1(x)\land g_2(x))\\
\hline
\end{array}
$$
\end{definition}
A simple induction would show the FOML formulas above are indeed equivalent to our $\LNES$ formulas. 
For $x\in \{T, S4, S5\}$, we write $\Sigma\models_{x-NES}\phi$ if for all $x$-model such that $\M,w\models_{NES} \Sigma$,  $\M,w\models_{NES} \phi$.

Here is an observation playing an important role in later proofs. 
%(proof sketch in Appendix \ref{app.prop.gnginvalid}).
\begin{proposition}\label{prop.gnginvalid}
For any $g\in\TermNES$, $\us{g}{\neg g}$ and $\es{g}{g}$ are both invalid over S5 models (thus also invalid over T, S4 models). 
\end{proposition}
\begin{proof}[Sketch]  First note that $\es{g}{g}$ is equivalent to $\neg \us{g}{\neg g}.$ We just need to show $\us{g}{\neg g}$ and its negation are both satisfiable for all $g$. Note that a model with a singleton domain $\{a\}$ can be viewed as a Kripke model for propositional modal logic, where a predicate $A$ can be viewed as a propositional letter: it holds on a world $w$ iff $a\in\rho_w(A)$. Then a term $g$ can be viewed as an equivalent modal formula. Since there is only one $a$ in the domain, $\us{g}{\neg g}$ is equivalent to $\neg g$, viewed as a modal formula, by the semantics. We just need to show each $\neg g$ and $g$ has singleton S5 models. It is easy to see that each $g$ and $\neg g$ (as modal formula) can be rewritten into an equivalent negative normal form (NNF) using $\K_i$ and $\dK_i$ to push the negation to the innermost propositional letter, e.g., $\neg \K_i\neg   \K_j \neg\K_i A$ can be rewritten as $\dK_i\K_j\dK_i\neg A.$ Now it is easy to satisfy such formulas by a Kripke model with a single world $w$ and the reflexive relations for all $R_i$: make $A$ true on $w$ iff the NNF of $g$ or $\neg g$ ends up with the literal $A$ instead of $\neg A$. Then we can turn this model into a first-order Kripke model by setting $\rho_w(A)=\{a\}$ iff $A$ is true on $w$. 
\end{proof}
% 
%     \begin{itemize}
%         \item $\M,w\models_{NES} \us{g_1}{g_2}$ if $\rho_w^+(g_1)\subseteq \rho_w^+(g_2)$
%         \item $\M,w\models_{NES} \es{g_1}{g_2}$ if $\rho_w^+(g_1)\cap \rho_w^+(g_2)\neq \emptyset$
%     \end{itemize}
%\subsection{Axiomization}
% \begin{itemize}
%     \item $\us{g}{g}$
%     \item $\us{g}{\neg \neg g}$,$\us{\neg \neg g}{g}$
%     \item $\us{Kg}{g}$ (T)
%     %\item $\us{Kg}{g}$ (T)
%     %\item $\us{Kg}{KKg}$ (S4)
%     %\item $\us{g}{K\dK g}$ (B)
%     %\item $\us{\dK g}{K \dK g}$ (S5)
% \end{itemize}
We propose the following proof system $\STNES$: 
\begin{center}
{\def\arraystretch{2.5}
\begin{tabular}{|l l l|}
\hline
 $\us{g}{g}, \qquad \us{K_ig}{g}$, & $\us{g}{\neg \neg g}, \qquad  \us{\neg \neg g}{g}$ &\\
    \AxiomC{$\us{g_1}{g_2}$}
    \AxiomC{$\us{g_2}{g_3}$}
    \RightLabel{Barbara}
    \BinaryInfC{$\us{g_1}{g_3}$}
    \DisplayProof & 
    \AxiomC{$\es{g_1}{g_2}$}
    \RightLabel{Conversion}
    \UnaryInfC{$\es{g_2}{g_1}$}
    \DisplayProof &
    \AxiomC{$\es{g_1}{g_2}$}
    \RightLabel{Existence}
    \UnaryInfC{$\es{g_1}{g_1}$}
    \DisplayProof \\ 
    \AxiomC{$\us{g}{\neg g}$}
    \RightLabel{Non-emptiness}
    \UnaryInfC{$\es{\neg g}{\neg g}$}
    \DisplayProof &     
    %     \AxiomC{[$\neg \phi$]}
    % \UnaryInfC{$\psi$}
    % \AxiomC{[$\neg \phi$]}
    % \UnaryInfC{$\neg\psi$}
    % \RightLabel{RAA}
    % \BinaryInfC{$\phi$}
    %   \DisplayProof 
          \AxiomC{[$\phi$]}
    \UnaryInfC{$\psi$}
    \AxiomC{[$\phi$]}
    \UnaryInfC{$\neg\psi$}
    \RightLabel{RAA}
    \BinaryInfC{$\neg \phi$}
      \DisplayProof &\AxiomC{$\vdash\us{g_1}{g_2}$}
    \RightLabel{K}
    \UnaryInfC{$\vdash\us{\K_i g_1}{\K_i g_2}$}
    \DisplayProof \\
    \hline 
\end{tabular}    
}
\end{center}

Clearly, $\us{K_ig}{g}$ is the counterpart of the usual T axiom in modal logic. The premise of Non-emptiness makes sure that nothing is $g$, since the FOML model has the nonempty domain, it follows that there is some $\neg g$. Note that the $K$-rule is restricted to provable formulas, as in the case of the monotonicity rule in modal logic. We define $\SFourNES$ to be $\STNES+\us{K_ig}{\K_i\K_ig}$, and $\SFiveNES$ to be $\SFourNES+ \us{\neg K_ig}{\K_i\neg\K_ig}$.  
It is straightforward to establish soundness if we read the formulas as their first-order modal counterparts. 
\begin{theorem}[Soundness]
    $\Sigma\vdash_{\STNES}\phi$ implies $\Sigma\models_{TNES}\phi$.  $\Sigma\vdash_{\SFourNES}\phi$ implies $\Sigma\models_{S4NES}\phi$. $\Sigma\vdash_{\SFiveNES}\phi$ implies $\Sigma\models_{S5NES}\phi$.
\end{theorem}
% \begin{proof}
%     We only prove the case of $K$-rule based on its FOML counterpart. 
%     %For instance of $T$, for arbitrary reflexive model $\M,w$, $wRw$ and hence $\rho^+_w(K_iA) = \bigcap_{wRv}\rho^+_v(A)\subseteq \rho^+_w(A)$. Hence $\models_{NES}\us{K_iA}{A}$.
%     If $\models \forall x(g_1(x)\to g_2(x))$, then by necessitation $\models \K_i \forall x(g_1(x)\to g_2(x))$. Since that models are constant-domain, $\models \forall x \K_i (g_1(x)\to g_2(x))$, thus by the K axiom for modal logic, we have $\models\forall x(\K_ig_1(x)\to \K_ig_2(x))$, which is $ \us{\K_ig_1}{\K_ig_2}$.
% \end{proof}
%\noteLYP{Reviewer 3 asks us to explain why these rules are important.}
Below are some derived rules and theorems that will play a role in the later proofs.
\begin{proposition}The following are derivable in $\STNES$ (and thus in $\SFourNES, \SFiveNES$). \\
 \label{prop.deriv}
{\def\arraystretch{2.5}
\begin{tabular}{ccc}
        \AxiomC{$\es{g_1}{g_2}$}
        \AxiomC{$\us{g_2}{g_3}$}
        \RightLabel{\textnormal{Darii}}
        \BinaryInfC{$\es{g_1}{g_3}$}
\DisplayProof&
        \AxiomC{$\us{g_1}{g_2}$}
        \RightLabel{\textnormal{Contrapositive}}
        \UnaryInfC{$\us{\neg g_2}{\neg g_1}$}
\DisplayProof &
        \AxiomC{$\us{g}{\neg g}$}
        \RightLabel{\textnormal{NonExistence}}
        \UnaryInfC{$\us{t}{\neg g}$}
        \DisplayProof
             \\
             \multicolumn{3}{c}{$\vdash_{\STNES}\us{g}{\dK_ig}$ \qquad  $\vdash_{\STNES}\us{\K_i g}{\dK_i g}$\qquad  $\vdash_{\STNES}\us{\K_i \neg \neg g}{\K_i g}$\qquad $\vdash_{\STNES}\us{\K_i g}{\K_i \neg\neg g}$}
\end{tabular}
}
\end{proposition}
\begin{proof}

\noindent Darii \qquad 
    \begin{prooftree}
        \AxiomC{$\us{g_2}{g_3}$}
        \AxiomC{$[\us{g_3}{\neg g_1}]$}
        \RightLabel{Barbara}
        \BinaryInfC{$\us{g_2}{\neg g_1}$}
        \AxiomC{$\es{g_1}{g_2}$}
        \RightLabel{Conversion}
        \UnaryInfC{$\es{g_2}{g_1}$}
        \RightLabel{RAA}
        \BinaryInfC{$\es{g_3}{g_1}$}
        \RightLabel{Conversion}
        \UnaryInfC{$\es{g_1}{g_3}$}
    \end{prooftree} 
Contrapositive 
\begin{prooftree}
    \AxiomC{$[\es{\neg g_2}{g_1}]$}
    \AxiomC{$\us{g_1}{g_2}$}
    \RightLabel{Darii}
    \BinaryInfC{$\es{\neg g_2}{g_2}$}
    \AxiomC{}{}
    \UnaryInfC{$\us{\neg g_2}{\neg g_2}$}
    \RightLabel{RAA}
    \BinaryInfC{$\us{\neg g_2}{\neg g_1}$}
\end{prooftree}
Non-Existence
\begin{prooftree}
    \AxiomC{$[\es{t}{\neg\neg g}]$}
    \AxiomC{}
    \RightLabel{}
    \UnaryInfC{$\us{\neg\neg g}{g}$}
    \RightLabel{Darii}
    \BinaryInfC{$\es{t}{g}$}
    \RightLabel{Conversion}
    \UnaryInfC{$\es{g}{t}$}
    \RightLabel{Existence}
    \UnaryInfC{$\es{g}{g}$}
    \AxiomC{$\us{g}{\neg g}$}
    \RightLabel{Darii}
    \BinaryInfC{$\es{g}{\neg g}$}
    \AxiomC{}
    \UnaryInfC{$\us{g}{g}$}
    \RightLabel{RAA}
    \BinaryInfC{$\us{t}{\neg g}$}
\end{prooftree}

$\us{g}{\dK_ig}$ can be proved based on the T-axiom $\us{\K_ig}{g}$ and Contrapositive above. $\us{\K_i g}{\dK_i g}$ follows by Barbara. 

    $\us{\K_i \neg \neg g}{\K_i g}$ and $\us{\K_i g}{\K_i \neg\neg g}$ can be shown by applying K principle on $\vdash_{\STNES}\us{g}{\neg\neg g}$ and $\vdash_{\STNES}\us{\neg \neg g}{g}$.
\end{proof}

%The proof can be found in Appendix \ref{app.prop.deriv}.

Recall that $\Sigma$ is inconsistent iff it can derive a contradiction. We can show: 
\begin{proposition}
    A set of formulas $\Sigma$ is inconsistent iff $\Sigma\vdash \es{g}{\neg g}$.
\end{proposition}
\begin{proof}
    $\Leftarrow:$ $\Sigma\vdash \us{g}{g}$ since it is an axiom. But by assumption, $\Sigma\vdash \neg \us{g}{g} = \es{g}{\neg g}$.
    $\Rightarrow:$ Without loss of generality, assume $\Sigma\vdash \es{g_1}{g_2},\us{g_1}{\neg g_2}$, then by conversion and Darii, $\Sigma\vdash \es{g_2}{\neg g_2}$.
\end{proof}

\section{Completeness} \label{sec.comp}
Now we proceed to prove that $\STNES$ is strongly complete w.r.t. reflexive frames. The result can be easily generalized to show the completeness of $\SFourNES$ and $\SFiveNES$ w.r.t.\ their corresponding classes of frames, to which we will come back at the end of the section. 

The completeness proof is based on the canonical (Kripke) model construction, similar to the case of modal logic. However, the language $\LNES$ is significantly weaker than the full language of FOML, which introduces some difficulties. In particular, $\LNES$ is essentially \textit{not} closed under subformulas: if we view our $\es{g_1}{g_2}$ and $\us{g_1}{g_2}$ as $\exists x (g_1(x)\land g_2(x))$ and $\forall x (g_1(x)\to g_2(x))$, then $g_1(x)$ and $g_2(x)$ are not expressible as \textit{formulas} in $\LNES$. Therefore in constructing the canonical model, we need to supplement each maximal consistent set $\Delta$ with a proper ``maximal consistent set'' of terms for each object, which can be viewed as a description of the object. Inspired by \cite{Moss11}, we define some notion of \textit{types} to capture such descriptions, which closely resembles the concept of \textit{points} in \cite{Moss11},\footnote{It is also called a \textit{quantum state} in \cite{Quantum99}.} in the setting of the orthoposet-based algebraic semantics for a (non-modal) syllogistic logic.\footnote{The completeness of the (non-modal) syllogistic logic in \cite{Moss11} was proved via a representation theorem of orthoposets. Our proofs below are self-contained and do not rely on the results of orthoposets. } 
%\noteLYP{Changed, and I added an explanation of the terminology used in footnote.}

Moreover, to prove the truth lemma eventually, we need Lemma \ref{lem.Econsis} which asserts that a set of existential sentences is consistent iff each single one of them is consistent. The lemma is equivalent to the assertion that in $\STNES$, every existential sentence brings no new universal consequences. The seemingly obvious statement is actually non-trivial since our system allows RAA and hence does not allow an easy inductive proof on deduction steps. We leave it to future work for finding an alternative direct proof system without RAA. For now, we need to construct a simpler canonical model to show Lemma \ref{lem.Econsis} in the coming subsection, which also leads to the weak completeness of $\STNES$. 

\subsection{Satisfiability of Existential Formulas and Weak Completeness}

Inspired by the notion of \textit{point} in \cite{Moss11}, we first define the \textit{types} as maximal descriptions of objects using terms. Obviously, an object must respect the universal formulas, and be either $g$ or not $g$ but not both for every term $g$. This will give us some basic properties of types. 
\begin{definition}[Type]
        A \emph{type} $\X$ is a subset of $Term^{NES}(U)$ s.t. 
    \begin{itemize}
        \item If $g_1\in \X$ and $\vdash_{\STNES} \us{g_1}{g_2}$, then $g_2\in\X$. (Respects Provably Barbara)
        \item For all $g\in Term^{NES}(U)$, either $g\in \X$ or $\neg g\in \X$. (Completeness)
        \item For all $g\in Term^{NES}(U)$, $g,\neg g$ are not both in $\X$. (Consistency)
    \end{itemize}
    Denote the set of all types by $\mathbb{W}$.
\end{definition}

\begin{definition}
    A collection $\mathcal{Y}$ of terms is said to be \emph{possible} if for all $g_1,g_2\in \mathcal{Y}$, $\not\vdash_{\STNES} \us{g_1}{\neg g_2}$.
\end{definition}
Note that all types are possible: If $g_1,g_2\in \X\in \mathbb{W}$ satisfies $\vdash_{\STNES} \us{g_1}{\neg g_2}$, then since $\X$ respects provably Barbara, $\neg g_2\in \X$, contradicting the consistency of $\X$.

\begin{lemma}[Witness Lemma for Possible Collection of Terms]
    If $\X_0$ is possible, then there is a type $\X\in \mathbb{W}$ extending it.\label{lem.witloose}
\end{lemma}

\begin{proof}
    Enumerate all terms in $\TermNES$ as $s_0, s_1,\dots$. We will construct a series of subsets of $\TermNES$, $\X_0\subseteq \X_1\subseteq \X_2 \dots$ s.t.
    \begin{itemize}
        \item For all $t_1,t_2\in \X_n$, $\not\vdash_{\STNES}\us{t_1}{\neg t_2}$. ($\X_n$ is possible)
        \item $\X_{n+1}$ is  $\X_n\cup\{s_{n+1}\}$ or $\X_n\cup\{\neg s_{n+1}\}$.
    \end{itemize}
    Now we show that each possible $\X_n$ can be extended into a possible $\X_{n+1}$. Given $\X_n$ that is possible, prove that at least one of $s_{n+1},\neg s_{n+1}$ can be added to $\X_n$ to form $\X_{n+1}$ that is possible. 
    %We jsut need to show if  $\vdash_{\STNES}\us{t}{s_{n+1}}$ for all $t\in \X_n$, or $\not\vdash_{\STNES}\us{t}{\neg s_{n+1}}$ for all $t\in \X_n$.
Assume $\X_n\cup\{s_{n+1}\}$ is not possible, then $\vdash_{\STNES}\us{t}{\neg t'}$ for some $t, t'\in \X\cup\{s_{n+1}\}$. We need to show that $\nvdash_{\STNES}\us{g}{\neg g'}$ for all $g, g'\in \X\cup\{\neg s_{n+1}\}$. Suppose not, then $\vdash_{\STNES}\us{g}{\neg g'}$ for some $g, g'\in \X\cup\{\neg s_{n+1}\}$. Since $\X_n$ is possible,  \textit{at least one} of $t$ and $t'$ must be $s_{n+1}$, and \textit{at least one} of $g$ and $g'$ must be $\neg s_{n+1}$. Furthermore, by Proposition~\ref{prop.gnginvalid} and soundness, $\nvdash_{\STNES}\us{u}{\neg u}$. Therefore \textit{exactly one of} $t$ and $t'$ is  $s_{n+1}$, and \textit{exactly one of} $g$ and $g'$ is $\neg s_{n+1}$. In the following we derive contradictions from $\vdash_{\STNES}\us{g}{\neg g'}$ and $\vdash_{\STNES}\us{t}{\neg t'}$ based on four cases. 

Let us consider the case when $t'=s_{n+1}$ and $g'=\neg s_{n+1}$, thus $t, g\in \X_n$. By double negation axiom, $\vdash_{\STNES}\us{g}{s_{n+1}}$ and $\vdash_{\STNES}\us{t}{\neg s_{n+1}}$. Then we have $\vdash_{\STNES}\us{t}{\neg g}$ by contrapositive and Barbara.
% the following derivation:
%     \begin{prooftree}
%         \AxiomC{}
%         \UnaryInfC{$\us{t}{\neg s_{n+1}}$}
%         \AxiomC{}
%         \UnaryInfC{$\us{g}{s_{n+1}}$}
%         \RightLabel{Contrapositive}
%         \UnaryInfC{$\us{\neg s_{n+1}}{\neg g}$}
%         \RightLabel{Barbara}
%         \BinaryInfC{$\us{t}{\neg g}$}
%     \end{prooftree}
Then it contradicts to the assumption that $\X_n$ is possible and we are done. The case when $t=s_{n+1}$ and $g=\neg s_{n+1}$ can be proved similarly using contrapostive and double negation. 
%Note that $t$ cannot be $s_n$ for $\vdash_{\STNES}\us{s_n}{\neg s_n}$ is not possible due to soundness. Similarly, $g$ cannot be $\neg s_n$. 

Now let us consider the case when $t=s_{n+1}$ and $g'=\neg s_{n+1}$, then we have 
$\vdash_{\STNES}\us{g}{s_{n+1}}$ and $\vdash_{\STNES}\us{s_{n+1}}{\neg t'}$. By Barbara, we have $\vdash_{\STNES}\us{g}{\neg t'}$, contradicting to the assumption that $\X_n$ is possible. Similar for the case when $t'=s_{n+1}$ and $g=\neg s_{n+1}$. 

    Consequently, at least one of $s_{n+1},\neg s_{n+1}$ can be added to $\X_n$ to form $\X_{n+1}$ that is possible.

 Let $\X = \bigcup_{n\in\mathbb{N}} \X_n$. Note that each $t\in \X$ has to be added or ``readded'' at some finite step $\X_k$ thus any two $t_1,t_2\in \X$ must be included in some $\X_j$. Therefore $\not\vdash_{\STNES}\us{t_1}{\neg t_2}$ since all the $\X_n$ are possible.

    Finally, we prove that $\X$ is a type. It is complete since one of $s_n,\neg s_n$ is added at some $\X_n$. It is consistent since if $t,\neg t\in \X$, but by axiom double negation we have $\vdash_{\STNES}\us{t}{\neg\neg t}$, contradicting the fact that $\X$ is possible. Now for provably Barbara: If $t_1\in \X$ and $\vdash_{\STNES}\us{t_1}{t_2}$, then $\vdash_{\STNES}\us{t_1}{\neg\neg t_2}$, hence $\neg t_2\not \in \X$ since $\X$ is possible. By its completeness, $t_2\in \X$.
\end{proof}

In the following, we build a canonical model for consistent sets of \textit{existential formulas}. Note that we use a fixed set $\mathbb{N}$ as the domain and assign a type to each number in $\mathbb{N}$ on each world, i.e., a world is simply a function from natural numbers to types. The accessibility relation is defined as usual in modal logic.  

% However, we cannot take types as objects in the domain directly, since we need a fixed object with varying properties in different worlds. For that purpose, we take $\mathbb{N}$ as the fixed domain, and assign natural numbers with different properties at each world. 
\begin{definition}[Canonical Model for Existential Formulas]
    The canonical model for existential formulas of $\STNES$ is defined as $\M^E = (W^E,\{R^E_i\}_{i\in I},D^E,\rho^E)$, where:

    \begin{itemize}
        \item $W^E = \mathbb{W}^\mathbb{N}$. That is: a world $w$ is a map from $\mathbb{N}$ to types.
        \item $w_1 R^E_iw_2$ iff $\K_ig\in w_1(n)$ entails $g\in w_2(n)$ for all $n\in \mathbb{N}$, $g\in Term^{NES}(U)$. % or maybe $\Delta^+_1\vdash_{EAS}\us{Z_\psi}{KA}$ entails $\Delta^+_2\vdash_{EAS}\us{Z_\psi}{A}$?
        \item $D^E = \mathbb{N}$
        \item $\rho^E(w,A) = \{n\mid A\in w(n)\}$.
    \end{itemize}
\end{definition}
\begin{proposition}[Reflexivity]
    The canonical model for existential formulas of $\STNES$ is reflexive.
\end{proposition}
\begin{proof}
    For arbitrary $g\in Term^{NES}(U)$, $w\in W^E$, if $\K_ig\in w(n)$, then since $\vdash_{\STNES}\us{\K_ig}{g}$, and $w(n)$ respects provably Barbara, $g\in w(n)$. Hence $w R_i^Ew$.
\end{proof}

To show that the canonical model satisfies the desired existential formulas, the key is to show that $\rho^{E^+}(w,g) = \{n\in \mathbb{N}\mid g\in w(n)\}$. That is: an object has property $g$ if $g$ is in the type it corresponds to. Similar to the proof of truth lemma in propositional modal logic, we have to prove an existence lemma for the induction step for $\K_i$. The existence lemma reads: if in $w$, an object is not known to be $g$, then $w$ must see a world where the object is not $g$.

\begin{lemma}[Existence Lemma]
    For all $w$, $m\in\mathbb{N}$, $t\in Term^{NES}(U)$ s.t. $\neg \K_it\in w(m)$, there is $w'$ s.t. $w R^E_iw'$ and $\neg t\in w'(m)$. 
\end{lemma}
\begin{proof}
    Consider the set $\mathcal{Y} = \{g\mid \K_ig\in w(m)\}\cup\{\neg t\}$. Prove that it is possible. Towards a contradiction, suppose  $\vdash_{\STNES}\us{t_1}{\neg t_2}$ for some $t_1,t_2\in \{g\mid \K_ig\in w(m)\}\cup\{\neg t\}$. By K principle we have $\vdash_{\STNES}\us{\K_it_1}{\K_i\neg t_2}$. There are three cases to be considered.
    
    Consider the case where $\K_it_1,\K_it_2\in w(m)$.   By Proposition \ref{prop.deriv}, $\vdash_{\STNES}\us{\K_i \neg t_2}{\dK_i\neg t_2}$ and Barbara, $\vdash_{\STNES}\us{\K_it_1}{\neg \K_it_2}$, contradicting the fact that $w(m)$ is possible. 

    Suppose $t_1=t_2=\neg t$, by Proposition \ref{prop.deriv}, $\vdash_{\STNES}\us{\K_i \neg t}{\K_i t}$ which entails $\vdash_{\STNES}\us{\K_i \neg t}{\dK_i t}$ but this is not possible by soundness, since it is not valid over T-models according to Proposition \ref{prop.gnginvalid}. 

    If $\K_it_1\in w(m)$ and $t_2 = \neg t$, then $\vdash_{\STNES} \us{t_1}{\neg \neg t}$ entails $ \vdash_{\STNES}\us{\K_it_1}{\K_i t}$, which contradicts the fact that $\neg \K_it\in w(m)$ and $w(m)$ is possible. If $\K_it_2\in w(m)$ and $t_1 = \neg t$, it leads to a contradiction as well since from $\vdash_{\STNES}\us{\K_it_1}{\neg \K_it_2}$, we have the symmetric $\vdash_{\STNES}\us{\K_it_2}{\neg \K_it_1}$ by contrapostive. 
    
    Consequently $\not\vdash_{\STNES} \us{t_1}{\neg t_2}$ for all $t_1,t_2\in \mathcal{Y}$. By Lemma~\ref{lem.witloose}, $\mathcal{Y} = \{g\mid \K_ig\in w(m)\}\cup\{\neg t\}$ can be extended to a type. Denote it by $\X_m$. Clearly, by repeating the reasoning in the above first case, for each $n\not=m\in \mathbb{N}$ we can find an $\X_n\in \mathbb{W}$ such that $\{g\mid \K_ig\in w(n)\}\in \X_n$. Let $w'$ then be defined by $w'(n) = \X_n$ for each $n$. Then $\neg t\in w'(m)$ and $w R_i^Ew'$. 
\end{proof}
\begin{lemma}[Truth Lemma for Terms]\label{lem.tle}
    $\rho^{E^+}(w,g) = \{m\mid g\in w(m)\}$ for all $g\in Term^{NES}(U)$.
\end{lemma}
\begin{proof}
    Apply an induction on terms. The base case is true by definition.

    Case 1: For $\neg g$, $\rho^{E^+}(w,\neg g) = D^E - \rho^{E^+}(w,g) = D^E - \{m\mid g\in w(m)\} = \{m\mid \neg g\in w(m)\}$. The last equality holds because types are consistent and complete.

    Case 2: For $\K_i g$, $\rho^{E^+}(w,\K_i g) =\bigcap_{wR_i^Ew'}\rho^{E^+}(w,g) = \bigcap_{wR_i^Ew'}\{m\mid g\in w(m)\}$, which equals $\{m\mid \K_ig\in w(m)\}$ by the following reasoning:

    $\supseteq$ side is easy to see, since if $m\in \{m\mid \K_ig\in w(m)\}$ and $wR_i^Ew'$, then $\K_ig\in w(m)$ entails $g\in w'(m)$ by definition. Hence $m\in \bigcap_{wR_i^Ew'}\{m\mid g\in w'(m)\}$.

    $\subseteq$ side. Assume $i\in \bigcap_{wR_i^Ew'}\{m\mid g\in w'(m)\}$, then $m\not\in \bigcup_{wR_i^Ew'}\{m\mid g\not\in w'(m)\}$, by the completeness and consistency of $w(m)$, $i\not\in \bigcup_{wR^E_iw'}\{m\mid \neg g\in w'(m)\}$. By Contrapositive of Existence Lemma, $m\not\in \{m\mid \neg \K_ig\in w(m)\}$. Consequently, $m\in \{m\mid \K_ig\in w(m)\}$.  
\end{proof}

Now we can show a set of consistent existential sentences is satisfiable. 

\begin{lemma}[Sets of Consistent  Existential Sentences are Satisfiable]\label{lem.Econsis}
    For a set of existential sentence $\Sigma_{Some}$, if $\not\vdash_{\STNES} \neg\phi_{Some}$ for all $\phi_{Some}\in \Sigma_{Some}$, then $\Sigma_{Some}$ is satisfiable (thus $\STNES$-consistent).
\end{lemma}
% \noteLYP{Notice that the condition is weaker than $\Sigma_{Some}$ is consistent, it only requires that each sentence in it is itsef consistent.}
\begin{proof}
    Enumerate sentences in $\Sigma_{Some}$ as $\phi_0, \phi_1, \dots$. For each $n$, suppose $\phi_n = \es{g_1}{g_2}$, we show that $\{g_1,g_2\}$ is possible. First note that since $\neg\phi$ is an abbreviation, the assumption $\not\vdash_{\STNES} \neg\phi_{Some}$ says $\not\vdash_{\STNES}\us{g_1}{\neg g_2}$. By contrapostive, $\not\vdash_{\STNES}\us{g_2}{\neg g_1}$. By Proposition~\ref{prop.gnginvalid}, $\us{g_1}{\neg g_1},\us{g_2}{\neg g_2}$ are not valid, thus cannot be proved in $\STNES$ by soundness. Therefore $\{g_1,g_2\}$ is possible and can be extended as a type by Lemma~\ref{lem.witloose}; call it $\X_n$. 
    
   Now we can define a $w\in W^E$. If $\Sigma_{Some}$ is infinite, let $w(n) = \X_n$ for all $n\in\mathbb{N}$; if not, let $w(n) = \X_n$ for $n\leq |\Sigma_{Some}|$, and $w(n) = \X_0$ for $n > |\Sigma_{Some}|$. Now we can show $\M^E,w\models_{NES} \Sigma_{Some}$ since each $\phi_n\in \Sigma_{some}$ is at least witnessed by $n$ due to our construction of $w$ and Lemma \ref{lem.tle}. Consistency of $\Sigma_{some}$ follows by soundness. 
\end{proof}

The weak completeness follows from the above lemma. 
\begin{theorem}[Weak Completeness]
    If $\models_{NES} \phi$, then $\vdash_{\STNES}\phi$.
\end{theorem}
\begin{proof}
   By Proposition~\ref{prop.gnginvalid} and the validity of the rule of Existence, we have $\not\models_{NES} \phi_{Some}$ for any existential sentence $\phi_{Some}$. Hence it suffices to prove that for all universal sentence $\phi_{All}$, if $\models_{NES} \phi_{All}$, then $\vdash_{\STNES}\phi_{All}$. Which is equivalent to showing if $\not\vdash_{\STNES} \phi_{All}$, then $\not\models_{NES} \phi_{All}$. Hence it suffices to show that for all existential sentence $\phi_{Some}$, if $\not\vdash_{\STNES} \neg\phi_{Some}$, then $\phi_{Some}$ is satisfiabe, which follows from Lemma~\ref{lem.Econsis} w.r.t.\ a singleton set.
\end{proof}

%\noteYW{Need to save some more space above to take the next definition on this page. }
\subsection{Strong Completeness}
Normally, a weak completeness result naturally leads to strong completeness if the logic is compact. However, even though $\LNES$ is indeed compact as it is a fragment of FOML, strong completeness does not easily follow and requires an argument based on Lemma \ref{lem.Econsis}. That is because in syllogistic, formulas are not closed under conjunction. Consequently, weak completeness does not lead to the satisfiability of every finite consistent formula set. Now we proceed to give a proof of strong completeness, again by building a (more complicated) canonical models, but for arbitrary maximal consistent sets.

Again, inspired by the notion of \textit{point} in \cite{Moss11}, we define the \textit{$\Delta$-type} to describe the sets of maximal properties an object may exemplify given the maximal consistent set  $\Delta$. 
%Obviously, the object must respect the universal formulas in $\Delta$, and be either $g$ or not $g$ but not both for every term $g$.
\begin{definition}[$\Delta$-type]
    Given an MCS $\Delta$, a \emph{$\Delta$-type}, denoted by $\X$ is a subset of $Term^{NES}(U)$ s.t. 
    \begin{itemize}
        \item If $g_1\in \X$ and $\Delta\vdash_{\STNES} \us{g_1}{g_2}$, then $g_2\in\X$. (Respects Barbara)
        \item For all $g\in Term^{NES}(U)$, either $g\in \X$ or $\neg g\in \X$. (Completeness)
        \item For all $g\in Term^{NES}(U)$, $g,\neg g$ are not both in $\X$. (Consistency)
    \end{itemize}
    Denote the set of all $\Delta$-types by $\mathbb{W}(\Delta)$.
\end{definition}

Given an existential sentence $\es{g_1}{g_2}\in \Delta$, we expect there to be some type $\X$ exemplifying both $g_1,g_2$. To show this,  we first generalize the notion of a \textit{possible} set of terms w.r.t.\ a maximal consistent set $\Delta$.
%\noteLYP{Changed and referred to Moss, previous writings are in brackets}

\begin{definition}
Given a maximal consistent set $\Delta$, call a set of terms $\mathcal{Y}$ \emph{$\Delta$-possible}, if for all $t_1,t_2\in \mathcal{Y}$, $\Delta\vdash_{\STNES}\es{t_1}{t_2}$.
\end{definition}
It is easy to see that the $\Delta$-types are $\Delta$-possible based on the fact that $\Delta$ is an MCS. The following lemma is the counterpart of Lemma 11.2 in \cite{Moss11} in the setting of orthoposet-based algebraic semantics. We present the following direct proof in our setting.  
%A direct proof in our setting is given in Appendix~\ref{app.lem.witness}.
\begin{lemma}[Witness Lemma for $\Delta$-Possible Collection] \label{lem.witness}
     Each set of terms $\X_0$ that is $\Delta$-possible can be extended to a $\Delta$-type $\X\in \mathbb{W}(\Delta)$. 
\end{lemma}
\begin{proof}
     Enumerate all terms in $Term^{NES}(U)-\X_0$ as $\{s_n\}$. Construct a series of subsets of $\TermNES$ s.t. $\X_0\subseteq \X_1\subseteq \X_2 \dots$ and:

    \begin{itemize}
        \item $\X_n$ is $\Delta$-possible: For all $t_1,t_2\in \X_n$, $\Delta\vdash_{\STNES}\es{t_1}{t_2}$.
        \item $\X_{n+1} = \X_n\cup\{g\}$, where $g = s_n$ or $\neg s_n$.
    \end{itemize}

    Now we show by induction that such a sequence can be constructed.
    
    By assumption $\X_0$ is $\Delta$-possible. 
    %\noteYW{$g_1, g_2?$}\noteLYP{Sorry again, I forgot to restate the Base Case when generalizing the lemma}

    Given $\X_n$ s.t. for all $t_1,t_2\in \X_n$, $\Delta\vdash_{\STNES}\es{t_1}{t_2}$, and $s_{n+1}$, prove that at least one of $s_{n+1},\neg s_{n+1}$ can be added to $\X_n$ to form $\X_{n+1}$ s.t. it remains $\Delta$-possible. Essentially, we have to show either (1) $\es{t}{s_{n+1}}\in\Delta$ for all $t\in \X_n$ and $\es{s_{n+1}}{s_{n+1}}\in\Delta$, or (2) $\es{t}{\neg s_{n+1}}\in\Delta$ for all $t\in \X_n$ and $\es{\neg s_{n+1}}{\neg s_{n+1}}\in\Delta$. 

    We prove that not (1) leads to (2). If (1) is not the case, there are two cases. Case 1: $\es{s_{n+1}}{s_{n+1}}\not\in\Delta$. Then $\us{s_{n+1}}{\neg s_{n+1}}\in \Delta$ since $\Delta$ is maximal. By derived rule nonexistence, $\us{g}{\neg s_{n+1}}\in\Delta$ for all $g\in Term^{NES}(U)$. For each $t\in \X_n$, $\es{t}{t}\in \Delta$, hence $\es{t}{\neg s_{n+1}}\in \Delta$ by Darii. For $\neg s_{n+1}$, by rule non-emptiness and that $\us{s_{n+1}}{\neg s_{n+1}}\in \Delta$, $\es{\neg s_{n+1}}{\neg s_{n+1}}\in \Delta$. Hence (2) holds.
    
    Case 2: Suppose $\es{t}{s_{n+1}}\not\in \Delta$ for some $t\in \X_n$, we need to show that (2) holds. For $\neg s_{n+1}$, if $\es{\neg s_{n+1}}{\neg s_{n+1}}\not\in \Delta$, then $\us{\neg s_{n+1}}{s_{n+1}}\in\Delta$.  Since $\es{t}{s_{n+1}}\not\in \Delta$, $\us{t}{\neg s_{n+1}}\in\Delta$, then $\us{t}{s_{n+1}}\in \Delta$ by Barbara, but since $t\in \X_n$, $\es{t}{t}\in \Delta$. This leads to $\es{t}{s_{n+1}}\in \Delta$, a contradiction to the assumption.
    We still need to show $\es{t'}{\neg s_{n+1}}\in\Delta$ for all $t'\in \X_n$. Assume towards a contradiction that $\es{t'}{\neg s_{n+1}}\not\in\Delta$ for some $t'\in \X_n$, then $\us{t'}{\neg \neg s_{n+1}}\in\Delta$. Since $\es{t}{s_{n+1}}\not\in \Delta$ then $\us{t}{\neg s_{n+1}}\in \Delta$. The following deduction shows that $\Delta\vdash_{\STNES}\us{t}{\neg t'}$, contradicting $\es{t}{t'}\in \Delta$, which follows from our induction assumption that $\X_n$ is $\Delta$-possible.

    \begin{prooftree}
        \AxiomC{$\us{t}{\neg s_{n+1}}$}
        \AxiomC{$\us{t'}{\neg\neg s_{n+1}}$}
        \AxiomC{}
        \UnaryInfC{$\us{\neg\neg s_{n+1}}{s_{n+1}}$}
        \RightLabel{Barbara}
        \BinaryInfC{$\us{t'}{s_{n+1}}$}
        \RightLabel{Contrapositive}
        \UnaryInfC{$\us{\neg s_{n+1}}{\neg t'}$}
        \RightLabel{Barbara}
        \BinaryInfC{$\us{t}{\neg t'}$}
    \end{prooftree}

    Consequently, either (1) or (2) holds and at least one of $s_{n+1},\neg s_{n+1}$ can be added to $\X_n$ to form $\X_{n+1}$ that is $\Delta$-possible. 

    \

    Let $\X = \bigcup_{n\in\mathbb{N}} \X_n$. Then $\Delta\vdash_{\STNES}\es{t_1}{t_2}$ for all $t_1,t_2\in \mathcal{\X}$.

    Finally, we prove that $\X$ is a $\Delta$-type. It is complete since one of $s_n,\neg s_n$ is added at each step, and all predicates in $U$ are eventually visited. It is consistent since $\Delta$ is consistent, so $\Delta\not\vdash_{\STNES} \es{t}{\neg t}$ for all $t$, hence $t,\neg t$ can't both be in $\X$. Finally we show that it respects Barbara: If $t_1\in \X$ and $\us{t_1}{t_2}\in \Delta$, then $\neg t_2\not \in \X$, otherwise we have $\es{t_1}{\neg t_2}\in\Delta$, contradicting the consistency of $\Delta$. By completeness, $t_2\in \X$.
\end{proof}

%(The proof is similar to the previous section and can be found in Appendix~\ref{app.lem.witness}. It can now be shown that each $\mathbb{W}(\Delta)$ is not empty, since either $\es{g}{g}\in\Delta$ or $\us{g}{\neg g}\in \Delta$. If the former, $\{g\}$ can be extended to a property set; if the latter, by rule non-emptiness it follows that $\es{\neg g}{\neg g}\in \Delta$ and $\{\neg g\}$ can be extended to a property set.)

%This proof can be easily generalized into a proof for the uncountable case by using Zorn's Lemma.

%\noteLYP{If we use the new proof, several proofs in the appendix should be restored into the main text, and the sentences in brackets may be delated}

Now we start to construct a canonical model for $\STNES$, and show that every maximal consistent set is satisfiable in it. Compared to the previous construction, we now need to take the maximal consistent sets (MCS) into consideration. A world $w$ is a pair of an MCS $\Delta$ and a map from $\mathbb{N}$ to $\mathbb{W}(\Delta)$.  By abusing the notation, as in the previous subsection, we write $w(m)$ for $f(m)$ if $w=\lr{\Delta, f}$.  
%We say $\Delta$ is the maximal consistent set behind $w$.
\begin{definition}[Canonical Model for $\STNES$]
    The canonical model for $\STNES$ is defined as $$\M^* = (W^*,\{R^*_i\}_{i\in I},D^*,\rho^*)$$ where:
    \begin{itemize}
        \item $W^* = \bigcup_{\Delta\in MCS}\{\lr{\Delta, f}\mid f\in \mathbb{W}(\Delta)^\mathbb{N}$\}. 
        \item $wR^*_iw'$ iff $\K_ig\in w(m)$ entails $g\in w'(m)$ for all $m\in \mathbb{N}$, $g\in Term^{NES}(U)$. % or maybe $\Delta^+_1\vdash_{EAS}\us{Z_\psi}{KA}$ entails $\Delta^+_2\vdash_{EAS}\us{Z_\psi}{A}$?
        \item $D^* = \mathbb{N}$
        \item $\rho^*(\lr{\Delta,f},A) = \{m\in\mathbb{N}\mid A\in f(m)\}$ for all $A\in U$.
    \end{itemize}
\end{definition}

%We show that since we are working under $T$ system, the canonical model defined above is reflexive.
It is not hard to show reflexivity as in the previous subsection.

\begin{lemma}[Reflexivity]
    The canonical model for $\STNES$ is reflexive.
\end{lemma}
\begin{proof}
    Take arbitrary $g\in Term^{NES}(U)$, $w\in W^*$. Assume $\Delta$ is the maximal consistent set behind $w$. If $\K_ig\in w(m)$, then since $\vdash_{\STNES}\us{K_ig}{g}$, $\us{K_ig}{g}\in\Delta$. Then $g\in w(m)$ since $w(m)$ respects Barbara. Hence $wR^*_iw$.
\end{proof}

% To show that the canonical model satisfies maximal consistent sets, the key is to show that $\rho^{*^+}(w,g) = \{m\in \mathbb{N}\mid g\in w(m)\}$. That is: an object has property $g$ if $g$ is in the property set it corresponds to. Similar to the proof of truth lemma in propositional modal logic, we have to prove an existence lemma for the induction step for $K$. The existence lemma reads: if in $w$, an object is not known to be $g$, then $w$ must see a world where the object is not $g$.
\begin{lemma}[Existence Lemma]
    For all $w$, $m\in \mathbb{N}$, $t\in Term^{NES}(U)$ s.t. $\neg \K_it\in w(m)$, there is $w'$ s.t. $wR_i^*w'$ and $\neg t\in w'(m)$. 
\end{lemma}
\begin{proof}
Assume $w$ = $\lr{\Delta,f}$ where $\Delta$ is a maximal consistent set. Consider $\Sigma =  \{\es{g}{\neg t}\mid \K_ig\in w(m)\}\cup\bigcup_{n\in \mathbb{N}}\{\es{g_1}{g_2}\mid \K_ig_1,\K_ig_2\in w(n)\}$, where the second part of the union is to  make sure we can obtain the right types. We show $\Sigma$ is consistent. Note that $\Sigma$ is made up of existential sentences only, thus by Lemma~\ref{lem.Econsis}, it suffices to prove that $\not\vdash_{\STNES}\neg \phi$ for all $\phi\in \Sigma$. 

    Given $\phi = \es{g}{\neg t}\in \Sigma$ for some $\K_ig\in w(m)$, assume for contradiction that $\vdash_{\STNES}\us{g}{t}$. Then by K principle, $\vdash_{\STNES}\us{\K_ig}{\K_it}$, hence $\us{\K_ig}{\K_it}\in \Delta$ and $\K_it\in f(m)$ since $f(m)$ respects Barbara, but $\neg \K_it\in w(m)$, contradicting consistency of $w(m)$.

    Given $\phi = \es{g_1}{g_2}\in \Sigma$ for some $\K_ig_1,\K_ig_2\in w(n)$, assume for contradiction $\vdash_{\STNES}\us{g_1}{\neg g_2}$. Again by K principle, $\vdash_{\STNES}\us{\K_ig_1}{\K_i\neg g_2}$. By Proposition~\ref{prop.deriv}, we have $\vdash_{\STNES}\us{\K_i \neg g_2}{\dK_i\neg g_2}$ and $\vdash_{\STNES}\us{\dK_i\neg  g_2}{\neg \K_i g_2}$. Now by Barbara, we have $\vdash_{\STNES}\us{\K_ig_1}{\neg \K_ig_2}$. Then $\us{\K_ig_1}{\neg \K_ig_2}\in \Delta$ and hence $\neg \K_ig_2\in w(n)$ since $w(n)$ respects Barbara. This is a contradiction to $\K_ig_2\in w(n)$ and that $w(n)$ is consistent.

    Since $\Sigma$ is consistent, we can expand $\Sigma$ to a maximal consistent set $\Delta'$ by a Lindenbaum-like argument. Observe that for all $n\neq m$,  $\K_ig_1,\K_ig_2\in w(n)$, we have $\es{g_1}{g_2}\in \{\es{g_1}{g_2}\mid \K_ig_1,\K_ig_2\in w(n)\} \subseteq \Delta'$, hence $\{g\mid \K_ig\in w(n)\}$ is $\Delta'$-possible and can be expanded to a $\Delta'$-type by Lemma~\ref{lem.witness}, denote it by $\X_n$. Similarly, as $\{g\mid \K_ig\in w(m)\}\cup\{\neg t\}$ is possible too, it can be expanded to a $\Delta'$-type, denote it by $\X_m$. Let $h$ be a function from $\mathbb{N}$ to $\mathbb{W}(\Delta')$ s.t. $h(n) = \X_n$. It is clear that $(\Delta,f)R_i^*(\Delta',h)$ and $\neg t\in h(m)$.
\end{proof}

Now we can establish the truth lemma similar to the one in the previous section. 
\begin{lemma}[Truth Lemma]\label{lem.tl}
    $\rho^{*^+}(w,g) = \{m\in\mathbb{N}\mid g\in w(m)\}$ for all $g\in Term^{NES}(U)$.
\end{lemma}
% \begin{proof}
%     Apply an induction on terms. The base case is true by definition.

%     Case 1: For $\neg g$, $\rho^{*^+}(w,\neg g) = D - \rho^{*^+}(w, g) = D - \{m\mid g\in w(m)\} = \{m\mid g\not\in w(m)\} = \{m\mid \neg g\in w(m)\}$. The last equality holds since property sets are consistent and complete.

%     Case 2: For $\K_i g$, $\rho^{*^+}(w, \K_ig) =\bigcap_{wR_i^*w'}\rho^{*^+}(w', g) = \bigcap_{wR_i^*w'}\{m\mid g\in w'(m)\}$. We show that it equals $\{m\mid \K_ig\in w(m)\}$.

%     $\supseteq$ side is easy to see, since if $m\in \{m\mid \K_ig\in w(m)\}$ and $wR_i^*w'$, then $\K_ig\in w(m)$ entails $g\in w'(m)$ by definition. Hence $m\in \bigcap_{wR_i^*w'}\{m\mid g\in w'(m)\}$.

%     $\subseteq$ side. Assume $m\in \bigcap_{wR_i^*w'}\{m\mid g\in w'(m)\}$, then $m\not\in \bigcup_{wR_i^*w'}\{m\mid g\not\in w'(m)\}$, by the completeness and consistency of $w(m)$, $m\not\in \bigcup_{wR^*_iw'}\{m\mid \neg g\in w'(m)\}$. By Contrapositive of Existence Lemma, $m\not\in \{m\mid \neg \K_ig\in w(m)\}$. Consequently, $m\in \{m\mid \K_ig\in w(m)\}$.
% \end{proof}
% The proof is similar to the truth lemma in the previous section and can be found in Appendix \ref{app.lem.tl}.

Finally, we can show the strong completeness of $\STNES$. 

\begin{theorem}[Strong Completeness for $\STNES$] \label{thm.scomp}
$\STNES$ is strongly complete w.r.t.\ the class of reflexive frames.  
\end{theorem}
\begin{proof}
As usual, we show each consistent $\Sigma$ for the $\STNES$ is satisfiable on a reflexive model.

    We first expand $\Sigma$ to a maximal consistent set $\Delta$, and enumerate the existential sentences in $\Delta$ as $\psi_0, \psi_1\dots $. For each $n$, suppose $\psi_n = \es{g_1}{g_2}$, $\{g_1,g_2\}$ is thus it is $\Delta$-possible since $\es{g_1}{g_2}$, $\es{g_2}{g_1}$, $\es{g_1}{g_1}$, $\es{g_2}{g_2}\in \Delta$ by rules of Conversion and Existence. Hence, it can be extended to a $\Delta$-type $\X_n$ in $\mathbb{W}(\Delta)$. Take $f:\mathbb{N}\to \mathbb{W}(\Delta)$ s.t. $f(n) = \X_n$ for all $n$. Show that $\M^*,\lr{\Delta, f} \models_{NES} \Delta$.
    
    For $\us{g_1}{g_2}\in \Delta$: Assume $n\in \rho_{w}^{*^+}(g_1)$, then by Truth Lemma $g_1\in w(n)$, then since $w(n)$ respects Barbara, $g_2\in w(n)$ hence by truth lemma $n\in \rho_{w}^{*^+}(g_2)$. Consequently $\M^*, w\models_{NES} \us{g_1}{g_2}$. 
    
    For $\es{g_1}{g_2}\in \Delta$: Suppose it is enumerated as $\phi_n$. By construction of $w$, $g_1,g_2\in \X_n = w(n)$. By truth lemma $n\in \rho_{\Delta,f}^{*^+}(g_1)\cap \rho_{\Delta,f}^{*^+}(g_2)$. Consequently $\M^*, w\models_{NES} \es{g_1}{g_2}$.
\end{proof}
%The proof can be found in Appendix \ref{app.thm.scomp}.

It is straightforward to adapt the completeness proof with extra axioms enforcing certain frame conditions in the canonical model. 
\begin{theorem}
    $\SFourNES$ and $\SFiveNES$ are strongly complete w.r.t.\ the class of reflexive and transitive frames and the class of frames with equivalence relations respectively.  
\end{theorem}

\section{Conclusions and future work} \label{sec.conc}
In this paper, we have taken the initial steps towards developing an epistemic syllogistic framework. We provided complete axiomatizations with respect to two epistemic syllogistic languages featuring \textit{de re} knowledge. The same techniques can be applied to \textit{belief} instead of knowledge. In fact, for systems concerning consistent belief over serial models, we only need to replace the counterpart of axiom T with D: $\us{Kg}{\dK g}$. Adding counterparts of axioms 4 and 5 will yield a complete system of KD45 belief. So far, the usual axioms can all enforce the canonical frame to adopt the desired structure as their modal logic counterparts. If we proceed without seriality, an additional rule is required: from $\es{Kg_1}{K\neg g_1}$, infer $\us{g_2}{Kg_3}$, to capture the scenario where the current world has no successor. It is evident that syllogisms can be studied in modal contexts other than the epistemic setting as well. 

As for other future work, we will consider the axiomatization problem of the full language of the so-called Aristotelian Modal Logic \cite{Protin22}, and also consider the \textit{de dicto} readings of the modal operators. It is also interesting to study the computational properties of these logics. One observation is that, like the cases of epistemic logics of know-wh \cite{Wang2018}, these epistemic syllogistic languages that we considered are one-variable fragments of FOML that are decidable in general. We will also explore the technical connections to other natural logics extending syllogistics such as \cite{kruckmanmoss2021}, and to the bundled fragments of first-order modal logic where quantifiers and modalities are also packed  to appear together \cite{Wang17d,PadmanabhaRW18}. 

\paragraph{Acknowledgement} This work is supported by NSSF grant 19BZX135 awarded to Yanjing Wang. The authors would like to thank Larry Moss for useful pointers and thank the three anonymous reviewers for their valuable comments that improved the presentation of the paper. 

\bibliographystyle{eptcs}
\bibliography{generic}

\appendix

\section{Proof Sketch of Theorem \ref{thm.es}} \label{app.thm.es}
\begin{proof}[Sketch]
    Note that for $\SEAS$, since there are $\LEAS$-formulas that cannot be negated syntactically, and we cannot equate $\Sigma\vdash_{\SEAS} \phi$ with that $\Sigma\cup\{\neg \phi\}\textit{ is a consistent set for }\SEAS$. Therefore we cannot reduce strong completeness to the satisfiability of any consistent set of formulas. 
    
    We leave the full proof for the extended version of this paper and only present a sketch here. Assume that $\Sigma$ is consistent, otherwise the conclusion is trivial. Separate $\Sigma$ into the non-modal part $\Sigma_0$ and the modal part $\Sigma_\K$. We consider all possible maximal consistent extension of $\Sigma_0$ in Assertoric Syllogistic and denote them by $\{\Delta_i\}_{i\in I}$. For each $\Delta_i\cup\Sigma_\K$, we construct a pointed model $\M_i,w_i = (W^i,R^i,D^i,\rho^i),w_i$ for it. 
    \begin{itemize}
    \item $W^i = \{w_i,v_0,v_1\}$.
    \item $R^i$ is the reflexive closure of $\{(w_i,v_0),(w_i,v_1)\}$.
    \item $D^i$ is $\Delta_i{_{Some+}}\sqcup \Delta'_i{_{Some+}}$, the positive existential sentences of the form $\es{A}{B}$ in $\Delta_i$ and its disjoint copy.
    \item $\rho^i_{w_i}(X) = \{\phi,\phi' \mid \phi=\es{A}{B}\textit{ and } \us{A}{X}\in \Delta_i\textit{ or }\us{B}{X}\in \Delta_i\}$. Where $\phi'$ is the copy of $\phi$.
    
    $\rho^i_{v_0}(X) = \{a\in D^i\mid \Delta^i\vdash \us{C}{KX} \textit{ for some C with } a\in \rho^i_w(C)\}\cup \{\phi=\es{B}{X}\in \Delta_i{_{Some+}}\mid \Delta^i\vdash \es{B}{KX}\}$.

    $\rho^i_{v_1}(X) = D^i - (\{a\in D^i\mid \Delta^i\vdash \us{C}{K\neg X} \textit{ for some C with } a\in \rho^i_w(C)\}\cup \{\phi'\in \Delta'_i{_{Some+}}\mid \phi = \es{B}{B},\Delta^i\vdash \es{B}{K\neg X}\})$.
    \end{itemize}
    The idea for the model is roughly the following: In the new world $v_0$, an object $a$ can have a property $X$ only if $a$ has property $C$ in the real world and $\Delta^i\cup\Sigma$ thinks $\us{C}{KX}$; or $\Delta^i\cup\Sigma$ thinks $\es{B}{KX}$ and $a$ happens to be $\es{B}{X}$. In the new world $v_1$, an object $a$ has every property $A$ unless $a$ has property $C$ and $\Delta^i\cup\Sigma$ thinks $\us{C}{K\neg X}$; or $\Delta^i\cup\Sigma$ thinks $\es{B}{K\neg X}$ and $a$ happens to be the copy of $\es{B}{B}$. We need a disjoint copy of $\Delta_i{_{Some+}}$ in the domain so that the mere fact that $\phi$ happens to have property $C$ does not validate a universal sentence.

    These models collectively describe all the possible models for $\Sigma$ under logical equivalence. Therefore, it can be shown that if $\Sigma\not\vdash_{\SEAS}\phi$, we can always find a $\Delta_i\supseteq \Sigma_0$ s.t. $\M_i,w\not\models \phi$. The collection of these models are called the canonical model family of $\Sigma$. Eventually, we will be able to prove that:
    \begin{enumerate}

        \item All models in the canonical model family satisfy $\Sigma$.
        \item If $\phi$ is satisfied by all models in the canonical model family of $\Sigma$, then $\Sigma\vdash_{\SEAS}\phi$.
    \end{enumerate}

    1. is standard practice. To show 2, we prove the converse: if $\Sigma\not\vdash_{\SEAS}\phi$ then there is one model in the canonical family that falsify it. As an example, we sketch the proof for case $\Sigma\not\vdash_{\SEAS}\us{A}{\K B}$, the other cases are similar. Consider $$\Sigma' = \Sigma_0\cup \{\phi\mid \K\phi\in \Sigma_K\}\cup\{\es{A}{\neg C}\mid \us{C}{\K B}\in \Sigma_K\}$$ 
    
    It can be shown to be consistent as a set of assertoric syllogistic. The main idea is that $\Sigma_0\cup \{\phi\mid \K\phi\in \Sigma_K\}$ is proof theoretic consequence of $\Sigma$, hence it is by assumption consistent. And if $\Sigma_0\cup \{\phi\mid \K\phi\in \Sigma_K\}$ deduces $\us{A}{C}$ for $\us{C}{\K B}\in \Sigma_K$, then $\Sigma\vdash_{\SEAS} \us{A}{\K B}$, which is a contradiction to the assumption. 
    
    Finally, $\Sigma'$ has a maximal consistent extension $\Delta^i$. It can be shown that the model $\M_i,w_i$ for $\Delta^i$ in the canonical model falsifies $\us{A}{\K B}$.
    
    After establishing 1 and 2, Completeness thus follows.
    %If $\phi$ is non-modal, then $\Sigma_0 \cup \{\phi\mid K\phi\in\Sigma\}\models_{AS} \phi$. Otherwise $\Sigma_0 \cup \{\phi\mid K\phi\in\Sigma\}\cup \{\neg\phi\}$ is satisfiable, thus consistent as a set of formulas of assertoric syllogistic. But by lemma 14 $\Sigma_0 \cup \{\phi\mid K\phi\in\Sigma\}\models_{AS} \phi$, a contradiction. Hence by completeness of AS, $\Sigma_0 \cup \{\psi\mid K\psi\in\Sigma\}\vdash_{AS} \phi$. $\Sigma\vdash_{ES} \phi$.
    %If $\phi$ is of $K$ modality, then take the canonical family $\M_i,w_i$ of $\Sigma$. By theorem 17, $\M_i,w_i\models_{ES}\Sigma$, hence $\M_i,w_i\models_{ES} \phi$ for models in the canonical family. And since $\phi$ is of $K$ modality, the lemma entails that $\Sigma\vdash_{ES} \phi$.
\end{proof}
\end{document}